\documentclass[runningheads]{llncs}

\usepackage{amsmath,amssymb}
\usepackage[numbers]{natbib}
\usepackage[T1]{fontenc}
\usepackage{xspace}
\usepackage{graphicx}
\usepackage{caption}
\usepackage[labelformat=simple]{subcaption}

\usepackage{booktabs}
\usepackage{tikz}
    \usetikzlibrary{positioning,calc,shadows,fit}
    \tikzset{l/.style={rectangle,draw,fill=blue!30!white},
        n/.style={l,rounded corners,fill=pink!40!white},
        st/.style={rectangle,draw,fill=green!40!white},
        i/.style={circle,inner sep=0mm},
        o/.style={circle,inner sep=0mm},
        c/.style={circle,fill=black,inner sep=.3mm},
        r/.style={circle,fill=black,inner sep=0.7mm},
        ra/.style={circle,draw,inner sep=0.7mm},
        arr/.style={->,>=stealth,very thick},
        b/.style={rectangle,draw},
        layer/.style={rectangle,draw,rounded corners},
        fit margins/.style={draw,dashed,/tikz/afit/.cd,#1,/tikz/.cd,inner xsep=-0.7mm,inner ysep=\pgfkeysvalueof{/tikz/afit/top}+\pgfkeysvalueof{/tikz/afit/bottom},yshift=-\pgfkeysvalueof{/tikz/afit/bottom}+\pgfkeysvalueof{/tikz/afit/top}},
        afit/.cd,bottom/.initial=1mm,top/.initial=1mm
    }
    \tikzstyle{level 1}=[level distance=7mm,sibling distance=17mm]
\usepackage{hyperref}

\colorlet{darkblue}{black!30!blue}
\colorlet{darkgray}{black!30!gray}
\colorlet{darkgreen}{green!50!black}

\newcommand{\figref}[1]{Fig.~\ref{#1}\xspace}
\newcommand{\R}{\ensuremath{\mathbb{R}}\xspace}
\newcommand{\Q}{\ensuremath{\mathbb{Q}}\xspace}
\newcommand{\nats}{\ensuremath{\mathbb{N}}\xspace}
\newcommand{\set}[1]{\{#1\}}
\newcommand{\NN}{\ensuremath{N}\xspace}
\newcommand{\plant}{\ensuremath{P}\xspace}
\newcommand{\NNCSit}[1][\plant,\NN]{\ensuremath{C_{#1}}\xspace}
\newcommand{\Ph}{\ensuremath{\mathcal{P}}\xspace}
\newcommand{\init}{\ensuremath{X_0}\xspace}
\newcommand{\target}{\ensuremath{\varphi}\xspace}
\newcommand{\idim}{\ensuremath{d}\xspace}
\newcommand{\cdim}{\ensuremath{c}\xspace}
\newcommand{\INC}[1]{\ensuremath{\texttt{INC}(#1)}\xspace}
\newcommand{\DEC}[1]{\ensuremath{\texttt{DEC}(#1)}\xspace}
\newcommand{\JZ}[2]{\texttt{JZ}(#1, #2)\xspace}
\newcommand{\pc}{\ensuremath{\mathit{pc}}\xspace}
\newcommand{\mach}{\mathcal{M}}
\newcommand{\stopp}{\texttt{STOP}}
\newcommand{\instr}{\texttt{I}}
\newcommand{\A}{\ensuremath{\mathcal{A}}\xspace}
\newcommand{\dec}{\ensuremath{\mathit{dec}}\xspace}
\DeclareMathOperator{\sign}{sign}
\newcommand{\wff}[1]{\ensuremath{\mathit{WF}_{#1}}\xspace}
\newcommand{\extend}[1]{\ensuremath{\widehat{#1}}\xspace}
\newcommand{\iter}{\ensuremath{I}\xspace}
\newcommand{\MMLM}{\ensuremath{\mathcal{H}}\xspace}
\newcommand{\flow}{\ensuremath{F}\xspace}
\newcommand{\grd}{\ensuremath{G}\xspace}
\newcommand{\UPh}{\ensuremath{\mathcal{FUP}}\xspace}

\title{The Reachability Problem for \\ Neural-Network Control Systems}

\author{%
Christian Schilling\inst{1}\orcidID{0000-0003-3658-1065}
\and
\\
Martin Zimmermann\inst{1}\orcidID{0000-0002-8038-2453}
}

\institute{Aalborg University, Aalborg, Denmark
\\
\email{\{christianms,mzi\}@cs.aau.dk}}

\begin{document}

\maketitle

\begin{abstract}
    A control system consists of a plant component and a controller which periodically computes a control input for the plant.
    We consider systems where the controller is implemented by a feedforward neural network with ReLU activations.
    The reachability problem asks, given a set of initial states, whether a set of target states can be reached.
    We show that this problem is undecidable even for trivial plants and fixed-depth neural networks with three inputs and outputs.
    We also show that the problem becomes semi-decidable when the plant as well as the input and target sets are given by automata over infinite words.
\end{abstract}

\section{Introduction}

Cyber-physical systems consist of digital (cyber) and physical components.
A common instance of this paradigm is a control system, consisting of a physical \emph{plant} and a \emph{controller} whose purpose is to steer the  plant to a desired behavior~\cite{DoyleFT13}.
Control theory studies the automatic synthesis of such controllers, which, for nonlinear systems, is a difficult task. 
Machine learning has long been successfully applied to tackle this task, where the learned controller was first represented by a (shallow) feedforward neural network~\cite{MillerSW95} and more recently by a \emph{deep neural network} (DNN)~\cite{LeGMD22}.
We call a control system with a DNN controller a \emph{neural-network control system} (NNCS).
Due to their black-box nature, DNNs have raised concerns about their correctness and safety, in particular in terms of the worst-case behavior~\cite{SzegedyZSBEGF13}.
However, as they are deployed in safety-critical applications, proving machine-learned NNCS correct is of utmost importance, and considerable resources have been invested into their verification~\cite{LiuALSBK21,LopezAFJLS23}.

In this paper, we are concerned with the fundamental problem of safety for NNCS: given a set of initial states and a set of bad states of the plant, does the controller prevent the plant to reach a bad state when started in an initial state?
Note that the failure of safety is captured by a reachability property: does there exist an initial state from which a bad state is reachable? 
Thus, in the following, we study the reachability problem for NNCS.

Recall that an NNCS is a combination of a DNN (the controller) and a plant. 
It is known that the reachability problem is already undecidable for sufficiently complex plants, even without any controller~\cite{KoiranM99}. 
So the question becomes: is there a simple but expressive class of plants for which the reachability problem is tractable?
Inspired by similar results for recurrent neural networks~\cite{SiegelmannS91,Hyotyniemi96}, we show in Section~\ref{sec:undecidable} that the answer is negative: the reachability problem is undecidable even for trivial plants. 
Intuitively, a DNN can simulate one computational step of a two-counter machine. 
Thus, a recurrent neural network can simulate a two-counter machine. 
As a DNN controlling a plant is essentially recurrent (as it bases its control decisions on the current state of the plant), undecidability follows.

On the positive side, we show in Section~\ref{sec:demi_decidable} that the reachability problem is at least semi-decidable for plants whose behavior can be captured by automata over infinite words:
Sälzer et al.\ showed that the behavior of DNNs can be captured by such automata~\cite{SalzerABL22}. 
Hence, relying on standard automata-theoretic constructions, the composition of a DNN and an automata-definable plant can also be captured by automata.
The class of automata-definable plants includes, for instance, plants that are described by multi-mode linear maps. Such maps are able to express, for example, the dynamics of adaptive cruise controls~\cite{LarsenMT15}.

\subsection{Related work}

Reachability in NNCS is generally challenging.
Existing approaches typically combine techniques developed for dynamical systems (the plant)~\cite{AlthoffFG21} and neural networks~\cite{LiuALSBK21}.
Tools such as CORA~\cite{Althoff15,KochdumperSAB23}, JuliaReach~\cite{BogomolovFFPS19,SchillingFG22}, and NNV~\cite{LopezCTJ23} compete in the ARCH-COMP friendly competition, and we refer to the report~\cite{LopezAFJLS23} for typical examples of NNCS.

Undecidability of questions about unbounded computations with piecewise-linear (PWL) functions is long known, e.g., periodicity in iterated 2D maps~\cite{Moore90} or reachability for linear hybrid automata~\cite{HenzingerKPV98}.
Similar results have been shown for DNNs.
Siegelmann and Sontag showed undecidability for unbounded computations in DNNs with activations given by a PWL approximation of the sigmoid function (which effectively is the ReLU function truncated at~$1$)~\cite{SiegelmannS91}.
Later, Hyotyniemi showed an encoding of two-counter machines in recurrent neural networks (RNNs) with ReLU activations~\cite{Hyotyniemi96}.
While an RNN can conceptually be seen as the special case of an NNCS without a plant, the formalism differs.
We thus consider our encoding of two-counter machines in NNCS of independent (yet mainly pedagogical) value.
Cabessa showed an encoding of two-counter machines in a variant of RNNs with conditional weights called \emph{spike-timing dependent plasticity}~\cite{Cabessa19}.
Recently, we also showed an encoding of two-counter machines in decision-tree control systems~\cite{SchillingLDL23}, where the DNN is replaced by a decision tree with simple conditions $x \leq c$ for some variable~$x$ and constant~$c$.

Katz et al.\ studied the problem of reachability in a ReLU DNN without iteration and showed that, given polyhedral (i.e., described by linear constraints) input and output sets, the reachability problem is NP-complete~\cite{KatzBDJK17}.
Sälzer and Lange later fixed some issues in the proof, mainly related to the effective representation of real numbers~\cite{SalzerL22}.

Sälzer et al.\ recently presented an encoding of a DNN in a weak Büchi automaton~\cite{SalzerABL22}.
We build on this encoding for the analysis of semi-decidability.

\section{Preliminaries}

We start by formally introducing the type of DNN under study.

\begin{definition}[Deep neural network]
    A \emph{neuron} is a function $\nu\colon \R^m \to \R$ with $\nu(\vec{x}) = \sigma(\sum_{i=1}^m c_i x_i + b)$, where $m$ is the input dimension, the $c_i \in \Q$ are the weights, $b \in \Q$ is the bias, and $\sigma \colon \R \to \R$ is the activation function of~$\nu$, which is either the identity function or the rectified linear unit (ReLU) $y \mapsto \max\set{y, 0}$.

    A \emph{layer} is a sequence of neurons $(\nu_1, \dots, \nu_n)$, all of the same input dimension $m$, computing the function $\ell \colon \R^m \to \R^n$ given by $\ell(\vec{x}) = (\nu_1(\vec{x}), \dots, \nu_n(\vec{x}))$.
    The dimensions $m$ and $n$ are the input resp.\ output dimension of the layer.

    A \emph{deep neural network} (DNN) $\NN$ is a sequence of layers $(\ell_1, \dots, \ell_k)$ such that the output dimension of $\ell_i$ is the input dimension of $\ell_{i+1}$ for all $i = 1, \dots, k-1$.
    The last layer is the output layer and all other layers are called hidden layers.
    If $m$ is the input dimension of $\ell_1$ and $n$ is the output dimension of $\ell_k$, then the DNN computes the function $\NN \colon \R^m \to \R^n$ defined as \[ \NN(\vec{x}) = \ell_k(\ell_{k-1}(\dots \ell_1(\vec{x})\dots)).\]
\end{definition}

\begin{figure}[tb]
    \centering
    \begin{tikzpicture}[box/.style={rectangle,draw,thick,rounded corners,drop shadow},t1/.style={thick},t/.style={t1,->,>=stealth}]
	\node[box,minimum height=6mm,fill=green!20!white] (nn) {Neural-network controller $\NN$};
	\node[box,fill=orange!30!white,below=6mm of nn,xshift=3cm] (plant) {\begin{tabular}{c}Plant $\plant$\end{tabular}};
	\node[left=of nn] (x0) {$\vec{x}_0$};
	\draw[t] (plant) -| node[left,near end] {$\vec{x}_k = \plant(\vec{x}_{k-1}, \vec{u}_k)$} (nn);
	\draw[t] (nn) -| node[right,near end] {$\vec{u}_k = \NN(\vec{x}_{k-1})$} (plant);
	\draw[t] (x0) -- (nn);
\end{tikzpicture}
    \caption{Neural-network control system.}
    \label{fig:nncs}
\end{figure}
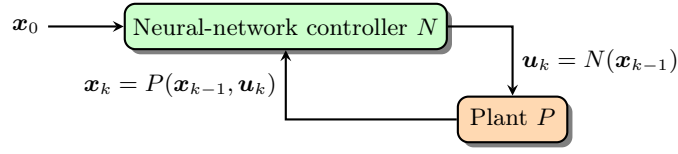

Next, we define control systems as depicted in \figref{fig:nncs}.
The system consists of a plant and a controller (here: a DNN) and acts in iterations: first, the controller computes a control input $\vec{u}$ for the plant based on the current state $\vec{x}$ of the plant. Next, from its inputs~$\vec{x}$ and $\vec{u}$, the plant computes a new state. Then the process repeats.
For the plant, we restrict ourselves to discrete time, i.e., we are only interested in its output and not its intermediate states.
For now, we also abstract from the concrete type of plant and just view it as a general function.

\begin{definition}[Neural-network control system]
    A neural-network control system (NNCS) is a tuple~$(\plant, \NN)$ with a plant $\plant \colon \R^{\idim+\cdim} \to \R^\idim$ and a controller given by a DNN~$\NN\colon \R^\idim \to \R^\cdim$, i.e., $\idim$ is the dimension of the states of $\plant$ and $\cdim$ is the dimension of the control vectors computed by $\NN$.

    The semantics of an NNCS are given as a sequence of states~$\vec{x}_k$ and control inputs~$\vec{u}_k$, induced by some initial state $\vec{x}_0 \in \R^\idim$ via
    \begin{align*}
        \vec{u}_k &= \NN(\vec{x}_{k-1}) \\
        \vec{x}_k &= \plant(\vec{x}_{k-1}, \vec{u}_k)
    \end{align*}
\end{definition}

We introduce a shorthand to express one iteration of the control loop in \figref{fig:nncs}, i.e., the composition of the DNN followed by the plant, to compute $\vec{x}_k$ from $\vec{x}_{k-1}$:
\[
    \NNCSit(\vec{x}_{k-1}) = \plant(\vec{x}_{k-1}, \NN(\vec{x}_{k-1}))
\]

We will focus on sets of states represented by linear constraints.
Given $a \in \Q^n$, $b \in \Q$, the set $H_{a,b} = \{\vec{x} \in \R^n \mid \langle a, \vec{x}\rangle \leq b\}$ is a linear constraint, where ``$\langle\cdot,\cdot\rangle$'' denotes the scalar product.
A polyhedron is a finite intersection of linear constraints.
Let $\Ph(n)$ denote the set of all polyhedra in $n$ dimensions.

We are now ready to define the reachability problem for NNCS.

\begin{problem}[Reachability problem for NNCS]\label{prob:reach}
    Given a DNN $\NN\colon \R^\idim \to \R^\cdim$, a plant $\plant\colon \R^{\idim+\cdim} \to \R^\idim$, a polyhedron~$\init \in \Ph(\idim)$ of initial states, and a polyhedron~$\target \in \Ph(\idim)$ of target states,
    does there exist an initial state $\vec{x}_0 \in \init$ and a $k \in \nats$ such that $(\NNCSit)^k(\vec{x}_0) \in \target$?
\end{problem}

\section{Undecidability}\label{sec:undecidable}

In this section, we prove that the NNCS reachability problem is undecidable.
The proof is by a reduction from the halting problem for two-counter machines.

\medskip

Formally, a two-counter machine~$\mach$ is a sequence
\[
(0:  \instr_0) (1:  \instr_1) \cdots (k-2:  \instr_{k-2})(k-1:  \stopp), 
\]
where the first element of a pair~$(\ell: \instr_\ell)$ is the line number and $\instr_\ell$ for $0 \le \ell < k-1$ is an instruction of the form
\begin{itemize}
    \item $\INC{i}$ with $i \in\set{0,1}$, 
    \item $\DEC{i}$ with $i \in\set{0,1}$, or
    \item $\JZ{i}{\ell'}$ with $i \in\set{0,1}$ and $\ell' \in \set{0, \ldots,k-1}$. 
\end{itemize}
A configuration of $\mach$ is of the form~$(\ell, c_0, c_1)$ with $\ell \in \set{0, \ldots, k-1}$ (the current value of the program counter) and $c_0, c_1\in\nats$ (the current contents of the two counters). 
The initial configuration is~$(0,0,0)$ and the unique successor configuration of a configuration~$(\ell,  c_0, c_1)$ is defined as follows:
\begin{itemize}
    \item If $\instr_\ell = \INC{i}$, then the successor configuration is $(\ell +1, c_0', c_1')$ with $c_i' = c_i +1$ and $c_{1-i}' = c_{1-i}$.
    \item If $\instr_\ell = \DEC{i}$, then the successor configuration is $(\ell +1, c_0', c_1')$ with $c_i' = \max\set{c_i -1,0}$ and $c_{1-i}' = c_{1-i}$.
    \item If $\instr_\ell = \JZ{i}{\ell'}$ and $c_i = 0$, then the successor configuration is $(\ell', c_0, c_1)$.
    \item If $\instr_\ell = \JZ{i}{\ell'}$ and $c_i > 0$, then the successor configuration is $(\ell+1, c_0, c_1)$.
    \item If $\instr_\ell = \stopp$, then $(\ell, c_0, c_1)$ has no successor configuration.
\end{itemize}
The unique run of $\mach$ (starting in the initial configuration) is defined as the maximal sequence~$\gamma_0 \gamma_1 \gamma_2 \cdots$ of configurations~$\gamma_j \in \nats^3$ where $\gamma_0$ is the initial configuration and $\gamma_{j+1}$ is the successor configuration of $\gamma_j$, if $\gamma_j$ has a successor configuration.
This run is either finite (line~$k-1$ is reached) or infinite (line~$k-1$ is never reached).
In the former case, we say that $\mach$ terminates.
The halting problem for two-counter machines asks, given a two-counter machine~$\mach$, whether $\mach$ terminates when started in the initial configuration.

\begin{proposition}[\cite{Minsky67}]
The halting problem for two-counter machines is undecidable.
\end{proposition}

In the following, we show that the halting problem for two-counter machines can be reduced to the reachability problem for NNCS by simulating the semantics of a two-counter machine by a NNCS.

\begin{theorem}\label{thm:undecidable}
    The reachability problem for NNCS is undecidable.
\end{theorem}

\begin{proof}
    Fix some two-counter machine $\mach$ with $k$ instructions.
    We show how to construct a gadget for each instruction of $\mach$ (except for the $\stopp$ instruction), which we then combine into a DNN simulating one configuration update of $\mach$. 
    Thus, the reachability problem for NNCS (which involves the iterated application of the DNN) then allows to simulate the full run of $\mach$.

    Formally, the DNN implements a function from $\R^3 \to \R^3$ with the following property: If the three inputs encode a non-stopping configuration of the two-counter machine, then the three outputs encode the successor configuration.
    Note that, since the weights and biases of the DNN we construct are integral, the outputs given integral inputs are also integral.
    In the following, we often implicitly assume that inputs are integral when we explain the intuition behind our construction.

    Our construction of the DNN fits into the common architecture~\cite{GoodfellowBC16} that all hidden neurons use ReLU activations and all output neurons use identity activations.
    We let the plant component simply turn the control input into the new state ($\plant(\vec{x}, \vec{u}) = \vec{u}$), as the DNN already simulates $\mach$.

    In some more detail, for every instruction~$(\ell; \instr_\ell)$ of $\mach$, we construct one gadget simulating this instruction. 
    All these gadgets will be executed in parallel in one iteration of the DNN, but only one of them (determined by the current value of the program counter) will actually perform a computation.
    The other gadgets just compute the identity function for each of their inputs.
    Thus, in the end we need to subtract $(k-2) \cdot v$ from each output $v$.

    All gadgets have inputs named $\pc$ (representing the current value of the program counter), and $c_0$ and $c_1$ (representing the current counter values), as well as three outputs named $\pc'$ (representing the value of the program counter of the successor configuration), and $c_0'$ and $c_1'$ (representing the  counter values of the successor configuration). 
    To simplify our construction, we use an additional gadget that conceptually checks whether the value of the program counter is equal to some fixed line number~$\ell$. This gadget is shown in \figref{fig:gadget_auxiliary}. The output~$a_\ell$ of this gadget (which has only one input~$\pc$) satisfies
    \[
        a_\ell = \begin{cases}
            1 & \pc = \ell, \\
            0 & \pc \neq \ell.
        \end{cases}
    \]

    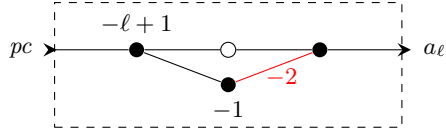
\begin{figure}[t]
        \centering
        \begin{tikzpicture}
	\node[i] (pc) {};
		\node[left=1mm of pc] {\pc};
	\node[r,right=1cm of pc] (n1) {};
		\node[above=0mm of n1] {$-\ell+1$};
	\node[ra,right=1cm of n1] (n2) {};
	\node[r,below=2.5mm of n2] (n3) {};
		\node[below=0mm of n3] {$-1$};
	\node[r,right=1cm of n2] (n4) {};
	\node[o,right=1cm of n4] (pco) {};
		\node[right=1mm of pco] {$a_\ell$};
	\node[fit margins={top=3mm,bottom=5mm},fit=(pc) (pco)] {};
	\draw (pc) -- (n1) -- (n2) -- (n4) -- (pco);
	\draw (n1) -- (n3);
	\draw[color=red] (n3) -- node[below,yshift=1mm,xshift=1mm]{$-2$} (n4);
	\foreach \x in {pc,pco} {
		\draw[arr] (\x.west) -- ($(\x.east)+(0.5mm,0mm)$);
	}
\end{tikzpicture}%
        \caption{Auxiliary gadget for instruction $\ell$. Here and in all later illustrations of DNNs, dots denote neurons, where filled dots use ReLU activations and empty dots can use either identity or ReLU activations (the choice is irrelevant since the value before the activation is nonnegative).
        Sometimes, as in this case, the empty dots are only present for a fully-connected architecture. Edge colors only serve the visual association with the weights. We omit weight~$1$ and bias~$0$ as well as connections with weight~$0$.}
        \label{fig:gadget_auxiliary}
    \end{figure}

    The outputs~$a_\ell$ of these auxiliary gadgets (we have one for each line number~$\ell$) will be fed into the other gadgets simulating the instructions.

    Next, we describe the instruction gadgets, where we restrict ourselves to the counter with index zero; the  counter with index one is treated in the analogous way.
    \figref{fig:gadgets} shows these gadgets together with the possible outputs.
    It is easy to verify that each gadget performs the corresponding computation whenever the input~$\pc$ is equal to $\ell$, and the identity function otherwise. 
    Let us stress that each gadget we construct depends both on the line number and the instruction.

    \begin{figure}[t]
        \centering
        \begin{subfigure}[b]{0.38\textwidth}
            \centering
            \begin{tikzpicture}
	\node[i] (c0) {};
		\node[left=1mm of c0] {$c_0$};
	\node[i,below=3mm of c0] (c1) {};
		\node[left=1mm of c1] {$c_1$};
	\node[i,below=3mm of c1] (pc) {};
		\node[left=1mm of pc] {\pc};
	\node[i,below=3mm of pc] (aj) {};
		\node[left=1mm of aj] {$a_\ell$};
	\node[ra,right=1cm of c0] (n1) {};
	\node[ra,right=1cm of c1] (n2) {};
	\node[ra,right=1cm of pc] (n3) {};
	\node[o,right=1cm of n1] (c0o) {};
		\node[right=1mm of c0o] {$c_0'$};
	\node[o,below=3mm of c0o] (c1o) {};
		\node[right=1mm of c1o] {$c_1'$};
	\node[o,below=3mm of c1o] (pco) {};
		\node[right=1mm of pco] {$\pc'$};
	\node[fit margins,fit=(c0o) (aj)] {};
	\draw (c0) -- (n1) -- (c0o);
	\draw (pc) -- (n3) -- (pco);
	\draw (aj) -- (n1);
	\draw (aj) -- (n3);
	\draw (c1) -- (n2) -- (c1o);
	\foreach \x in {c0,c1,pc,aj,c0o,c1o,pco} {
		\draw[arr] (\x.west) -- ($(\x.east)+(0.5mm,0mm)$);
	}
\end{tikzpicture}%
            \caption{$\ell: \INC{0}$ gadget.}
            \label{fig:gadget_inc}
        \end{subfigure}
        \hfill
        \begin{subfigure}[b]{0.60\textwidth}
            \centering
            \begin{tabular}{l@{\hspace*{3mm}}l@{\hspace*{3mm}}l@{\hspace*{3mm}}l}
                Condition & $x'$ & $y'$ & $\pc'$ \\
                \midrule
                $\pc = \ell$ & $c_0 + 1$ & $c_1$ & $\pc + 1$ \\
                $\pc \neq \ell$ & $c_0$ & $c_1$ & $\pc$ \\
            \end{tabular}
            \vspace*{3mm}
            \caption{$\ell: \INC{0}$ gadget's possible output values.}
            \label{fig:table_inc}
        \end{subfigure}
        \\[7mm]
        \begin{subfigure}[b]{0.38\textwidth}
            \centering
            \begin{tikzpicture}
	\node[i] (c0) {};
		\node[left=1mm of c0] {$c_0$};
	\node[i,below=3mm of c0] (c1) {};
		\node[left=1mm of c1] {$c_1$};
	\node[i,below=3mm of c1] (pc) {};
		\node[left=1mm of pc] {\pc};
	\node[i,below=3mm of pc] (aj) {};
		\node[left=1mm of aj] {$a_\ell$};
	\node[r,right=1cm of c0] (n1) {};
	\node[ra,right=1cm of c1] (n2) {};
	\node[ra,right=1cm of pc] (n3) {};
	\node[o,right=1cm of n1] (c0o) {};
		\node[right=1mm of c0o] {$c_0'$};
	\node[o,below=3mm of c0o] (c1o) {};
		\node[right=1mm of c1o] {$c_1'$};
	\node[o,below=3mm of c1o] (pco) {};
		\node[right=1mm of pco] {$\pc'$};
	\node[fit margins,fit=(c0o) (aj)] {};
	\draw (c0) -- (n1) -- (c0o);
	\draw (pc) -- (n3) -- (pco);
	\draw[color=red] (aj) -- node[above,near end,xshift=-2mm,yshift=-1mm] {$-1$} (n1);
	\draw (aj) -- (n3);
	\draw (c1) -- (n2) -- (c1o);
	\foreach \x in {c0,c1,pc,aj,c0o,c1o,pco} {
		\draw[arr] (\x.west) -- ($(\x.east)+(0.5mm,0mm)$);
	}
\end{tikzpicture}%
            \caption{$\ell: \DEC{0}$ gadget.}
            \label{fig:gadget_dec}
        \end{subfigure}
        \hfill
           \begin{subfigure}[b]{0.60\textwidth}
            \centering
            \begin{tabular}{l@{\hspace*{3mm}}l@{\hspace*{3mm}}l@{\hspace*{3mm}}l}
                Condition & $c_0'$ & $c_1'$ & $\pc'$ \\
                \midrule
                $\pc = \ell$ & $\max\set{c_0-1,0}$ & $c_1$ & $\pc + 1$ \\
                $\pc \neq \ell$ & $c_0$ & $c_1$ & $\pc$ \\
            \end{tabular}
            \vspace*{3mm}
            \caption{$\ell: \DEC{0}$ gadget's possible output values.}
            \label{fig:table_dec}
        \end{subfigure}
        \\[7mm]

        \begin{subfigure}[b]{0.38\textwidth}
            \centering
            \begin{tikzpicture}
	\node[i] (c0) {};
		\node[left=1mm of c0] {$c_0$};
	\node[i,below=3mm of c0] (c1) {};
		\node[left=1mm of c1] {$c_1$};
	\node[i,below=3mm of c1] (pc) {};
		\node[left=1mm of pc] {\pc};
	\node[i,below=3mm of pc] (aj) {};
		\node[left=1mm of aj] {$a_\ell$};
	\node[ra,right=1cm of c0] (n11) {};
	\node[ra,right=1cm of n11] (n12) {};
	\node[ra,right=1cm of n12] (n13) {};
	\node[ra,right=1cm of c1] (n21) {};
	\node[ra,right=1cm of n21] (n22) {};
	\node[ra,right=1cm of n22] (n23) {};
	\node[ra,right=1cm of pc] (n1) {};
	\node[ra,at=(aj-|n1)] (n4) {};
	\node[r,below=2mm of n4] (n5) {};
		\node[below=0mm of n5] {$1$};
	\node[ra,right=1cm of n1] (n2) {};
	\node[r,right=1cm of n5] (n6) {};
		\node[below=0mm of n6] {$-1$};
	\node[r,right=1cm of n2] (n3) {};
	\node[o,right=1cm of n3] (pco) {};
		\node[right=1mm of pco] {$\pc'$};
	\node[o,at=(c1-|pco)] (c1o) {};
		\node[right=1mm of c1o] {$c_1'$};
	\node[o,at=(c0-|c1o)] (c0o) {};
		\node[right=1mm of c0o] {$c_0'$};
	\node[fit margins={bottom=5mm},fit=(c0o) (aj)] {};
	\draw (c0) -- (n11) -- (n12) -- (n13) -- (c0o);
	\draw[red] (c0) -- node[near start,above,xshift=2mm,yshift=-.5mm] {$-1$} (n5);
	\draw (c1) -- (n21) -- (n22) -- (n23) -- (c1o);
	\draw (pc) -- (n1) -- (n2) -- (n3) -- (pco);
	\draw (aj) -- (n1);
	\draw (aj) -- (n4) -- (n6);
	\draw (n5) -- (n6);
	\draw[blue] (n6) -- node[right,yshift=-1mm] {$\ell'-\ell-2$} (n3);
	\foreach \x in {c0,c1,pc,aj,c0o,c1o,pco} {
		\draw[arr] (\x.west) -- ($(\x.east)+(0.5mm,0mm)$);
	}
\end{tikzpicture}%
            \caption{$\ell: \JZ{0}{\ell'}$ gadget.}
            \label{fig:gadget_jz}
        \end{subfigure}
        \hfill
        \begin{subfigure}[b]{0.60\textwidth}
            \centering
            ~~~~~~~~
            \begin{tabular}{l@{\hspace*{3mm}}l@{\hspace*{3mm}}l@{\hspace*{3mm}}l}
                Condition & $c_0'$ & $c_1'$ & $\pc'$ \\
                \midrule
                $\pc = \ell \text{ and } c_0 = 0$ & $c_0$ & $c_1$ & $\ell'$ \\
                $\pc = \ell \text{ and } c_0 \neq 0$ & $c_0$ & $c_1$ & $\pc + 1$ \\
                $\pc \neq \ell$ & $c_0$ & $c_1$ & $\pc$ \\
            \end{tabular}
            \vspace*{6mm}
            \caption{$\ell: \JZ{0}{\ell'}$ gadget's possible output values.}
            \label{fig:table_jz}
        \end{subfigure}
        \caption{The gadgets for the three instructions. See \figref{fig:gadget_auxiliary} for further explanations.}
        \label{fig:gadgets}
    \end{figure}
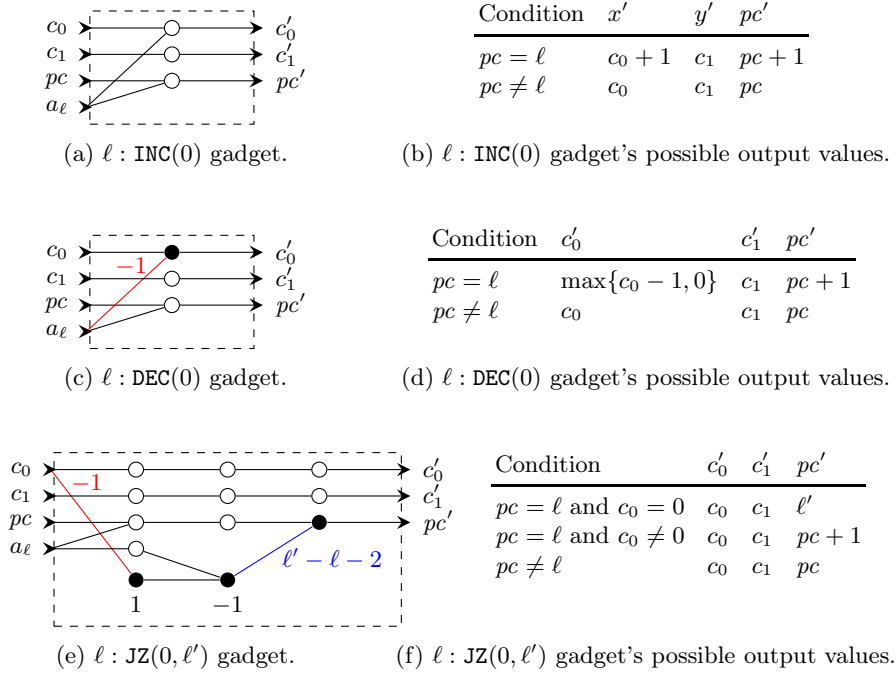

    The final layout of the gadgets is shown in \figref{fig:complete_construction}.
    Essentially, each auxiliary gadget is wired to the corresponding instruction gadget, and at the end we need to subtract the inputs $k-2$ times as described above.
    Note that there is no gadget for the $\stopp$ instruction (instruction~$k-1$ in $\mach$).
    When $\pc$ is equal to $k-1$, then the DNN computes the identity function: First, all $a_\ell$ are equal to $0$; Hence, each of the $k-1$ instruction gadgets~$I_\ell$ computes the identity function; after the subtraction, we are indeed left with the identity.

    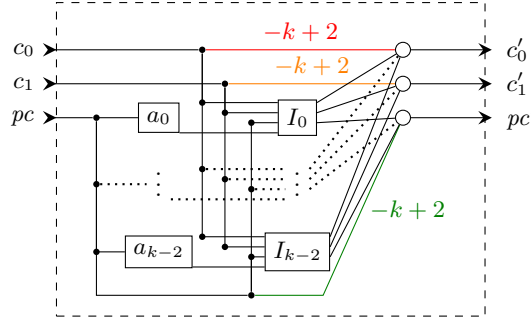
\begin{figure}[t]
        \centering
        \begin{tikzpicture}
	\node[i] (c0) {};
		\node[left=1mm of c0] {$c_0$};
	\node[i,below=4mm of c0] (c1) {};
		\node[left=1mm of c1] {$c_1$};
	\node[i,below=4mm of c1] (pc) {};
		\node[left=1mm of pc] {\pc};
	\node[c,right=5mm of pc] (pcaux1) {};
	\node[b,right=5mm of pcaux1] (a1) {$a_0$};
	\node[below=2mm of a1] (dotsa) {\raisebox{2mm}{$\vdots$}};
	\node[b,below=2mm of dotsa] (ak) {$a_{k-2}$};
	\coordinate[below=23mm of pcaux1] (pcaux2);
	\node[c,right=20mm of pcaux2] (pcaux3) {};
	\coordinate[right=9mm of pcaux3] (pcaux4);
	\node[c,right=19mm of c0] (c0aux) {};
	\node[c,right=22mm of c1] (c1aux) {};
	\node[b,right=13mm of a1] (i1) {$I_0$};
	\node[at=(dotsa-|i1)] (dotsi) {\raisebox{2mm}{$\vdots$}};
	\node[b,at=(ak-|i1)] (ik) {$I_{k-2}$};
	\node[ra,right=25mm of c0aux] (c0r) {};
	\node[ra,at=(c0r|-c1)] (c1r) {};
	\node[ra,at=(c0r|-pc)] (pcr) {};
	\node[o,right=1cm of c0r] (c0o) {};
		\node[right=1mm of c0o] {$c_0'$};
	\node[o,at=(c0o|-c1)] (c1o) {};
		\node[right=1mm of c1o] {$c_1'$};
	\node[o,at=(c0o|-pc)] (pco) {};
		\node[right=1mm of pco] {$\pc'$};
	\node[fit margins={top=3mm,bottom=13mm},fit=(c0o) (pc)] {};
	\draw (c0) -- (c0aux);
	\draw[red] (c0aux) -- node[above] {$-k+2$} (c0r);
	\draw (c0aux) |- node[c] {} ($(i1.west) + (0mm,2mm)$);
	\coordinate (c0aux3) at=($(dotsi.west) + (0mm,2mm)$);
	\node[c,at=(c0aux|-c0aux3)] (c0aux2) {};
	\draw[dotted,thick] (c0aux2) -- (c0aux3);
	\draw (c0aux) |- node[c] {} ($(ik.west) + (0mm,2mm)$);
	\draw (c1) -- (c1aux);
	\draw[orange] (c1aux) -- node[above] {$-k+2$} (c1r);
	\draw (c1aux) |- node[c] {} ($(ik.west) + (0mm,0.66mm)$);
	\coordinate (c1aux3) at=($(dotsi.west) + (0mm,0.66mm)$);
	\node[c,at=(c1aux|-c1aux3)] (c1aux2) {};
	\draw[dotted,thick] (c1aux2) -- (c1aux3);
	\draw (c1aux) |- node[c] {} ($(i1.west) + (0mm,0.66mm)$);
	\draw (pc) -- (pcaux1) -- (pcaux2) -- (pcaux3);
	\draw (pcaux1) |- (a1);
	\node[c,at=(pcaux1|-dotsa)] (pcaux5) {};
	\draw[dotted,thick] (pcaux5) -- (dotsa);
	\draw (pcaux1) |- node[c] {} (ak);
	\draw (pcaux3) |- node[c] {} ($(i1.west) + (0mm,-0.66mm)$);
	\coordinate (pcaux7) at=($(dotsi.west) + (0mm,-0.66mm)$);
	\node[c,at=(pcaux3|-pcaux7)] (pcaux6) {};
	\draw[dotted,thick] (pcaux6) -- (pcaux7);
	\draw (pcaux3) |- node[c] {} ($(ik.west) + (0mm,-0.66mm)$);
	\draw[darkgreen] (pcaux3) -- (pcaux4) -- node[right] {$-k+2$} (pcr);
	\draw ($(a1.east) + (0mm,-2mm)$) -- ($(i1.west) + (0mm,-2mm)$);
	\draw[dotted,thick] ($(dotsa.east) + (0mm,-2mm)$) -- ($(dotsi.west) + (0mm,-2mm)$);
	\draw ($(ak.east) + (0mm,-2mm)$) -- ($(ik.west) + (0mm,-2mm)$);
	\draw ($(i1.east) + (0mm,2mm)$) -- (c0r);
	\draw ($(i1.east) + (0mm,0.66mm)$) -- (c1r);
	\draw ($(i1.east) + (0mm,-0.66mm)$) -- (pcr);
	\draw[dotted,thick] ($(dotsi.east) + (0mm,2mm)$) -- (c0r);
	\draw[dotted,thick] ($(dotsi.east) + (0mm,0.66mm)$) -- (c1r);
	\draw[dotted,thick] ($(dotsi.east) + (0mm,-0.66mm)$) -- (pcr);
	\draw ($(ik.east) + (0mm,2mm)$) -- (c0r);
	\draw ($(ik.east) + (0mm,0.66mm)$) -- (c1r);
	\draw ($(ik.east) + (0mm,-0.66mm)$) -- (pcr);
	\draw (c0r) -- (c0o);
	\draw (c1r) -- (c1o);
	\draw (pcr) -- (pco);
	\foreach \x in {c0,c1,pc,c0o,c1o,pco} {
		\draw[arr] (\x.west) -- ($(\x.east)+(0.2mm,0mm)$);
	}
\end{tikzpicture}%
        \caption{Complete construction. Each box represents an auxiliary gadget~$a_\ell$ resp.\ an instruction gadget $I_\ell$. Small dots denote junctions of connections and have no further semantics. The last layer is the output layer (with identity activations).}
        \label{fig:complete_construction}
    \end{figure}

    Finally, the initial input to the DNN is $\vec{x}_0 = (0, 0, 0)$ (representing the initial configuration) and the target set is $\target = \{(k-1, c_0, c_1) \mid c_0,c_1\ge 0\}$, where $k-1$ is the last instruction number ($\stopp$) of $\mach$.
    Clearly, $\mach$ terminates if and only if the NNCS reaches a state satisfying $\target$ when started in $X_0=\set{\vec{x}_0}$. \qed
\end{proof}

We note that the DNNs simulating two-counter machines are rather simple.

\begin{corollary}
    The NNCS reachability problem remains undecidable for DNNs with integral weights, $3$ input and output dimensions, $6$ hidden layers, a singleton initial set, and a target set~$o = v$ for some output neuron~$o$ and constant~$v \in \nats$.
\end{corollary}

\begin{figure}[t]
    \centering
    \begin{subfigure}[b]{0.45\textwidth}
        \centering
        \begin{tikzpicture}
	\node[i] (c01) {};
		\node[left=1mm of c01] {$c_0^1$};
	\node[i,below=3mm of c01] (c11) {};
		\node[left=1mm of c11] {$c_1^1$};
	\node[i,below=3mm of c11] (pc1) {};
		\node[left=1mm of pc1] {$\pc^1$};
	\node[below=24mm of pc1] (dotsin) {\!\!\!$\vdots$};
	\node[i,below=4mm of dotsin] (i17) {};
		\node[left=1mm of i17] {$i_1^7$};
	\node[i,below=3mm of i17] (i27) {};
		\node[left=1mm of i27] {$i_2^7$};
	\node[i,below=3mm of i27] (i37) {};
		\node[left=1mm of i37] {$i_3^7$};
	%
	\node[right=1cm of c01] (l11) {};
	\node[right=1cm of c11] (l12) {};
	\node[right=1cm of pc1] (l13) {};
	\node[layer,fit=(l11) (l13)] (l1) {\raisebox{-2mm}{$\ell_1$}};
	\node[at=(dotsin-|l1)] {$\vdots$};
	\node[right=1cm of i17] (l71) {};
	\node[right=1cm of i27] (l72) {};
	\node[right=1cm of i37] (l73) {};
	\node[layer,fit=(l71) (l73)] (l7) {\raisebox{-2mm}{$\ell_7$}};
	%
	\node[o,right=1cm of l11] (o11) {};
		\node[right=1mm of o11] {$o_0^1$};
	\node[o,right=1cm of l12] (o12) {};
		\node[right=1mm of o12] {$o_1^1$};
	\node[o,right=1cm of l13] (o13) {};
		\node[right=1mm of o13] {$o_2^1$};
	\node[o,below=5mm of o13] (o21) {};
		\node[right=1mm of o21] {$o_0^2$};
	\node[o,below=3mm of o21] (o22) {};
		\node[right=1mm of o22] {$o_1^2$};
	\node[o,below=3mm of o22] (o23) {};
		\node[right=1mm of o23] {$o_2^2$};
	\node[o,below=3mm of o23] (ovdots) {};
		\node[right=1mm of ovdots] {$\vdots$};
	\node[o,below=3mm of ovdots] (o2k) {};
		\node[right=1mm of o2k] {$o_{k}^2$};
	\node[at=(dotsin-|o23)] {~~$\vdots$};
	\node[o,right=1cm of l71] (o71) {};
		\node[right=1mm of o71] {$o_0^7$};
	\node[o,right=1cm of l72] (o72) {};
		\node[right=1mm of o72] {$o_1^7$};
	\node[o,right=1cm of l73] (o73) {};
		\node[right=1mm of o73] {$o_2^7$};
	\node[fit margins={top=2mm,bottom=2mm},fit=(o11) (i37)] {};
	%
	\draw (c01) -- (l11-|l1.west);
	\draw (l11-|l1.east) -- (o21.west);
	\draw (c11) -- (l1);
	\draw (l1.east) -- (o22.west);
	\draw (pc1) -- (l13-|l1.west);
	\draw (l13-|l1.east) -- (o23.west);
	\draw[dotted,thick] (l13-|l1.east) -- (ovdots);
	\draw (l13-|l1.east) -- (o2k.west);
	\draw (i17) -- (l71-|l7.west);
	\draw[red] (l71-|l7.east) -- (o11.west);
	\draw (i27) -- (l7);
	\draw[red] (l7.east) -- (o12.west);
	\draw (i37) -- (l73-|l7.west);
	\draw[red] (l73-|l7.east) -- (o13.west);
	\coordinate (a1) at ($(l71-|l7.east)+(0,15mm)$);
	\coordinate (a2) at ($(l72-|l7.east)+(0,15mm)$);
	\coordinate (a3) at ($(l73-|l7.east)+(0,15mm)$);
	\draw[dotted,thick] (a1) -- (o71.west);
	\draw[dotted,thick] (a2) -- (o71.west);
	\draw[dotted,thick] (a3) -- (o71.west);
	\draw[dotted,thick] (a1) -- (o72.west);
	\draw[dotted,thick] (a2) -- (o72.west);
	\draw[dotted,thick] (a3) -- (o72.west);
	\draw[dotted,thick] (a1) -- (o73.west);
	\draw[dotted,thick] (a2) -- (o73.west);
	\draw[dotted,thick] (a3) -- (o73.west);
	\foreach \x in {c01,c11,pc1,i17,i27,i37,o11,o12,o13,o21,o22,o23,o2k,o71,o72,o73} {
		\draw[arr] (\x.west) -- ($(\x.east)+(0.5mm,0mm)$);
	}
\end{tikzpicture}%
        \caption{Stacking of previous hidden layers.}
        \label{fig:collapse_full}
    \end{subfigure}
    \hspace*{1cm}
    \begin{subfigure}[b]{0.35\textwidth}
        \centering
        \begin{tikzpicture}
	\node[i] (i1) {};
		\node[left=1mm of i1] {$m_i^1$};
	\node[i,below=3mm of i1] (i2) {};
		\node[left=1mm of i2] {$m_i^2$};
	\node[below=3mm of i2] (dotsin) {\!\!\!$\vdots$};
	\node[i,below=4mm of dotsin] (i7) {};
		\node[left=1mm of i7] {$m_i^7$};
	%
	\node[ra,right=1cm of i1] (l1) {};
	\node[ra,right=1cm of i2] (l2) {};
	\node[at=(dotsin-|l2)] {$\vdots$};
	\node[ra,right=1cm of i7] (l7) {};
	\node[above=3mm of l7] (l6) {};
	%
	\node[o,right=1cm of l1] (o1) {};
		\node[right=1mm of o1] {$m_o^1$};
	\node[o,right=1cm of l2] (o2) {};
		\node[right=1mm of o2] {$m_o^2$};
	\node[below=3mm of o2] (o3) {};
	\node[at=(dotsin-|o2)] {~~$\vdots$};
	\node[o,right=1cm of l7] (o7) {};
		\node[right=1mm of o7] {$m_o^7$};
	\node[fit margins={top=2mm,bottom=2mm},fit=(o1) (i7)] {};
	%
	\draw (i1) -- (l1) -- (o2.west);
	\draw (i2) -- (l2);
	\draw[dotted,thick] (l2) -- (o3);
	\draw (i7) -- (l7) -- (o1.west);
	\draw[dotted,thick] (l6) -- (o7.west);
	\foreach \x in {i1,i2,i7,o1,o2,o7} {
		\draw[arr] (\x.west) -- ($(\x.east)+(0.5mm,0mm)$);
	}
\end{tikzpicture}%
        \vspace*{15mm}
        \caption{Modulo-$7$ counter.}
        \label{fig:collapse_modulo}
    \end{subfigure}
    \caption{Construction with a single hidden layer.}
    \label{fig:collapse}
\end{figure}
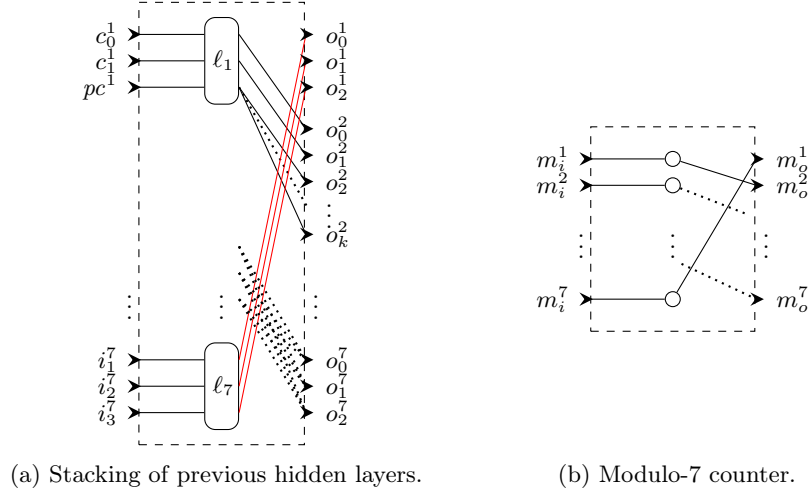

One may wonder whether the six hidden layers are necessary.
In general, one cannot hope to obtain a small neural network when removing layers~\cite{AroraBMM18}.
However, since we can iterate the NNCS, and the plant model is not interfering, we can reduce one iteration of a DNN $\NN$ with six hidden layers (constructed in the proof above) to seven iterations of a DNN~$\NN'$ with one hidden layer.\footnote{Typically, neural networks with only one hidden layer are not called \emph{deep}.}
\figref{fig:collapse} shows a sketch of the construction idea.
Essentially, we take the hidden layers of $\NN$ and stack them as one wide hidden layer in $\NN'$.
(For instance, layer~$\ell_1$ has width $k + 1$.)
We refer to each of these hidden layers as a \emph{track}.
For each track, we need to add input and output dimensions corresponding to the number of neurons in the respective previous and next hidden layers.
The output of track~$j$ is fed to track~$j + 1$ (and the output of the last track is fed to the first track).

When presented with an input vector $\vec{x}_0$ of appropriate size, the first track performs the computation of the first hidden layer and feeds its output to the second track, and so on.
After seven iterations, the output of the last track will equal the output of the seven-layer DNN $\NN$ after the first iteration.
This output is then used as the input of the first track again and the process continues.

Finally, we need to make sure that the other tracks do not accidentally produce an output that leads to a target state between multiples of seven iterations.
In order to only consider outputs in every seventh iteration, we use the additional gadget shown in \figref{fig:collapse_modulo}.
This gadget has seven inputs and outputs and is to be stacked below the other DNN.
When the initial input is $(1, 0, 0, 0, 0, 0, 0)$, the $1$ is propagated to the second index, and so on, until it arrives back at the first index after seven iterations.

The target set $\target$ now simply needs to get extended to arbitrary values in the auxiliary dimensions, except for the seven-last entry ($m_i^1$), which has to equal one.
Formally: $\target' = (\target \land m_i^1 \leq 1 \land -m_i^1 \leq -1)$.

In summary, by scaling the number of inputs and outputs with $\mach$, we obtain a DNN with one (wide) hidden layer.

\begin{corollary}\label{cor:one_hidden_layer}
    The NNCS reachability problem remains undecidable for DNNs with integral weights, one hidden layer, a singleton initial set, and a target set~$o = v$ for some output neuron~$o$ and constant~$v \in \nats$.
\end{corollary}

\section{Semi-decidability}\label{sec:demi_decidable}

In this section, we show that the NNCS reachability problem is semi-decidable for a particular class of plants.
Indeed, from a single initial state $\vec{x}_0$, we can enumerate all states $\NNCSit(\vec{x}_0)^k$ reachable in $k$ iterations and for each of them decide membership in the target polyhedron $\target$.
However, since we allow for an initial set $\init$, this algorithm is not effective.

The image of a polyhedron under a ReLU DNN is a (finite) union of polyhedra~\cite{MontufarPCB14}.
If we choose a class of plants with the same property, we obtain an effective algorithm again.
In what follows, we show a more general result by using an automaton encoding of DNNs from~\cite{SalzerABL22}.
This will allow us to more abstractly consider a class of plants that is definable in the same automaton formalism.

\medskip

We slightly deviate from the original approach by Sälzer et al.~\cite{SalzerABL22} in that we use a more expressive automaton model, as we are (unlike them) not bothered with efficiency considerations (since our problem is undecidable).

\begin{definition}[Büchi automaton]
    A (nondeterministic) \emph{Büchi automaton} (NBA) $\A = (Q, \Sigma, q_0, \delta, F)$ consists of a finite set~$Q$ of states, a finite alphabet~$\Sigma$, an initial state $q_0 \in Q$, a transition relation $\delta \subseteq Q \times \Sigma \times Q$, and a set of accepting states $F \subseteq Q$.

    A run on an infinite word $w = a_0 a_1 \dots$ is an infinite sequence of states $q_0, q_1, \dots$ starting in the initial state and satisfying $(q_i, a_i, q_{i+1}) \in \delta$ for all $i \geq 0$. A run is accepting if $q_i \in F$ for infinitely many $i$. The language of $\A$ is
    \[
        L(\A) = \{w \in \Sigma^\omega \mid \text{$A$ has an accepting run on $w$}\}.
    \]

    A language is $\omega$-regular if there exists an NBA that accepts it.
\end{definition}

In the following, we recall an effective encoding of real numbers in NBA from~\cite{SalzerABL22}.
Let $\Sigma = \{+, -, 0, 1, .\}$.
A word $w = s a_n \dots a_0 . b_0 b_1 \dots$ with $n \geq 0$, $s \in \{+, -\}$, $a_i, b_i \in \{0, 1\}$ encodes the real value
\[
    \dec(w) = (-1)^{\sign(s)} \cdot \left( \sum_{i=0}^n a_i \cdot 2^i + \sum_{i=0}^\infty b_i \cdot 2^{-(i+1)} \right)
\]
where $\sign(s) = 0$ if $s = +$ and $\sign(s) = 1$ if $s = -$.
As usual, the word encoding is not unique, but the decoding is~\cite[Page 5]{SalzerABL22}.

Now, we switch to a word encoding of multiple numbers by using a product alphabet.
A symbol over this product alphabet $\Sigma^k$ is a $k$-vector of symbols.
A word over $\Sigma^k$ is well-formed if it is of the form
\[
    w =
    \begin{bmatrix} s_1 \\ \vdots \\ s_k \end{bmatrix}
    \begin{bmatrix} a_{1,n} \\ \vdots \\ a_{k,n} \end{bmatrix}
    \cdots
    \begin{bmatrix} a_{1,0} \\ \vdots \\ a_{k,0} \end{bmatrix}
    \begin{bmatrix} . \\ \vdots \\ . \end{bmatrix}
    \begin{bmatrix} b_{1,0} \\ \vdots \\ b_{k,0} \end{bmatrix}
    \begin{bmatrix} b_{1,1} \\ \vdots \\ b_{k,1} \end{bmatrix}
    \cdots
\]
where $s_i \in \{+, -\}$, $a_{i,j}, b_{i,h} \in \{0, 1\}$ for $i = 1, \dots, k$, $j = 0, \dots, n$, and $h = 0, 1, \dots$.
In other words, the signs and the point are aligned, which can be achieved by filling up with leading zeros.
The language~$\wff{k}$ of well-formed words is $\omega$-regular~\cite{SalzerABL22}.
The selection of a single component $i \in \{1, \dots, k\}$ is obtained in the obvious way:
\[
    w_i = s_i a_{i,n} \dots a_{i,0} \, . \, b_{i,0} b_{i,1} \dots
\]

If an NBA over $\Sigma^k$ accepts only well-formed words, then we can understand its language as a relation over $\R^k$.
Furthermore, linear constraints are also $\omega$-regular~\cite{SalzerABL22}.
Thus, as NBA are closed under intersection and union, (finite unions of) polyhedra are also $\omega$-regular.
Finally, we can also use NBAs to encode functions~$f\colon \R^m \rightarrow \R^n$ via their graphs, which are relations over~$\R^{m+n}$.

Sälzer et al.\ showed that every function computable by a DNN can be represented by an NBA.\footnote{They actually proved the result for the more restrictive class of eventually-always weak NBA. But for us it is more prudent to consider the more general class of NBA.}

\begin{proposition}[Theorem~1 in~\cite{SalzerABL22}]
\label{prop:nn2NBA}
    Let $\NN\colon \R^\idim \to \R^\cdim$ be a DNN.
    There exists an NBA $\A_{\NN}$ over $\Sigma^{\idim+\cdim}$ with
    \begin{equation*}
        L(\A_{\NN}) = \{w \in \!\wff{\!\idim+\cdim} \mid \NN(\dec(w_1), \dots, \dec(w_\idim)) \!=\! (\dec(w_{\idim+1}), \dots, \dec(w_{\idim+\cdim}))\}.
    \end{equation*}
\end{proposition}

For our application, we need to slightly modify the automaton~$\A_\NN$ from Proposition~\ref{prop:nn2NBA} so that it also copies its input for further use. This modification can be implemented by replacing each transition label~$(a_1, \ldots, a_n, a_1',\dots, a_m')$ by $(a_1, \ldots, a_n,a_1, \ldots, a_n, a_1',\dots, a_m')$.

\begin{corollary}\label{cor:nn2NBA}
    Let $\NN\colon \R^\idim \to \R^\cdim$ be a DNN.
    There exists an NBA $\extend{\A}_{\NN}$ over $\Sigma^{\idim+\idim+\cdim}$ with
    \begin{align*}
        L(\extend{\A}_{\NN}) = \{w \in {}& \wff{\idim+\idim+\cdim} \mid 
        (w_1, \dots, w_\idim) = (w_{\idim+1}, \dots, w_{\idim+\idim}) \text{ and } \\
        & \NN((\dec(w_1), \dots, \dec(w_\idim)) = (\dec(w_{\idim+\idim+1}), \dots, \dec(w_{\idim+\idim+\cdim}))\}.
    \end{align*}
\end{corollary}

Thus, the semantics of DNNs can be captured by NBAs.
This is in general not true for plants.
Hence, in the following, we restrict ourselves to plants that can also be captured by NBAs.

\begin{definition}[$\omega$-regular plant]\label{def:plant2NBA}
    A plant $\plant\colon \R^{\idim+\cdim} \to \R^{\idim}$ is \emph{$\omega$-regular} if there exists an NBA $\A_{\plant}$ over $\Sigma^{\idim+\cdim+\idim}$ such that
    \begin{align*}
        L(\A_{\plant}) = \{ w \in {}& \wff{\idim+\cdim+\idim} \mid {} \\
        & \plant(\dec(w_1), \dots, \dec(w_{\idim+\cdim})) = (\dec(w_{\idim+\cdim+1}), \dots, \dec(w_{\idim+\cdim+\idim}))\}.
    \end{align*}
\end{definition}

Now, both the DNN and the plant are given by NBA. Hence, we can apply standard automata-theoretic constructions to capture a bounded number of applications of the control loop by repeatedly composing the NBA for the DNN and the NBA for the plant. 
To this end, we introduce the (parametric) composition operator $\circ_k$ constructing from two NBAs $\A_1$ and $\A_2$, which accept the graphs of two functions $f_1\colon \R^{k_1} \to \R^k$ and $f_2\colon \R^k \to \R^{k_2}$, an NBA $\A_1 \circ_k \A_2$ accepting the graph of $\vec{x} \mapsto f_2(f_1(\vec{x}))$.

\begin{lemma}[Lemma 4 of \cite{SalzerABL22}]  
Let $k,k_1,k_2 \ge 0$ and let $A_1$ and $\A_2$ be two NBAs over $\Sigma^{k_1+k}$ and $\Sigma^{k+k_2}$, respectively.
Then, there exists an NBA~$\A_1 \circ_k \A_2$ over $\Sigma^{k_1 + k_2}$ accepting the language
\begin{align*}
    \{ (u_1, \dots, u_{k_1}, w_{k +1},\dots, w_{k+ k_2}) \mid {}&{} \exists (v_1, \dots, v_k) \text{ s.t. } \\
    & (u_1, \dots, u_{k_1}, v_1, \dots, v_k) \in L(\A_1)  \text{ and }\\
    &(v_1, \dots, v_k, w_{k+1},\dots, w_{k+k_2}) \in L(\A_2)\}.
\end{align*}
\end{lemma}

Now we are ready to prove our main result of this section: NNCS reachability restricted to $\omega$-regular plants is semi-decidable. Note that this is tight, as the problem is undecidable as shown in Theorem~\ref{thm:undecidable}: the plant just returning its control input as output is $\omega$-regular.

\begin{theorem}\label{thm:semi_decidable}
    The NNCS reachability problem is semi-decidable when restricted to $\omega$-regular plants.
\end{theorem}

\begin{proof}
    We are given a problem instance $(\NN, \plant, \init, \target)$ and need to (semi)-decide whether there exists a $k \ge 0$ such that $(\NNCSit)^k(\vec{x}_0) \in \target$ for some $\vec{x}_0 \in \init$. 
    Let $\idim$ be the dimension of the states of $\plant$ and $\cdim$ be the dimension of the control vectors computed by $\NN$, respectively.

    Let $\A_\NN$ (over $\Sigma^{\idim+\idim+\cdim}$) and $\A_\plant$ (over $\Sigma^{\idim+\cdim +\idim}$) be the NBAs as in Corollary~\ref{cor:nn2NBA} and Definition~\ref{def:plant2NBA}. Then,
    we define $\iter_0$ to be an NBA accepting the graph of the $\idim$-ary identity function
    \[
        L(\iter_0) = \set{w \in \wff{\idim+\idim} \mid (w_0, \dots, w_\idim) = (w_{\idim+1},\dots, w_{\idim+\idim})}
    \]
    and, for $k \ge 1$, $\iter_k = \iter_{k-1} \circ_{\idim} (\extend{\A}_\NN \circ_{\idim+\cdim} \A_\plant)$.

    By construction, we have
    \[
        (w_1, \dots, w_{\idim}, w_{\idim+1}, \dots, w_{\idim+\idim}) \in L(\iter_k)
    \]
    if and only if
    \[
        (\dec(w_{\idim+1}), \dots, \dec(w_{\idim+\idim})) \in (\NNCSit)^k(\dec(w_1), \dots, \dec(w_{\idim})).
    \]

    There are NBAs~$\A_0$ and $\A_\target$ accepting $\init$ and $\target$, as they are polyhedra.
    Both these NBAs have alphabet~$\Sigma^\idim$, while each $\iter_k$ has alphabet~$\Sigma^{\idim+\idim}$ where the first $\idim$ components encode the inputs and the last $\idim$ components encode the outputs.
    Hence, to restrict $\A_0$ and $\A_\target$ to $\init$ and $\target$, we need to widen $\A_0$ and $\A_\target$ to NBA with alphabet~$\Sigma^{\idim+\idim}$.
    Formally, let NBAs~$\extend{\A}_{0}$ and $\extend{\A}_{\target}$ (both over $\Sigma^{\idim+\idim}$) such that
    \begin{itemize}
        \item $L(\extend{\A}_0)$  contains the encodings of all vectors~$(x_1, \dots, x_{\idim+\idim})\in\wff{\idim+\idim}$ such that $(x_1, \dots, x_\idim)$ is in $\init\subseteq \R^\idim$ and $(x_{\idim+1}, \dots, x_{\idim+\idim}) \in \R^\idim$ is arbitrary, and  
        \item $L(\extend{\A}_\target)$  contains the encodings of all vectors~$(x_1, \dots, x_{\idim+\idim})\in\wff{\idim+\idim}$ such that $(x_1, \dots, x_\idim) \in \R^\idim$ is arbitrary and $(x_{\idim+1}, \dots, x_{\idim+\idim})$ is in $\target \subseteq \R^\idim$.
    \end{itemize}
    Now, there exist an $\vec{x}_0 \in \init$ and a $k\ge 0$ such that $(\NNCSit)^k(\vec{x}_0) \in \target$ if and only if the language of $\extend{\A}_0 \cap \iter_k \cap \extend{\A}_\target$ is nonempty. 

    In summary, to semi-decide the NNCS reachability problem for $\omega$-regular plants, we iteratively construct $\iter_k$ for $k\ge0$ and check $\extend{\A}_0 \cap \iter_k \cap \extend{\A}_\target$ for nonemptiness.
    \qed
\end{proof}

Let us remark that the construction in Theorem~\ref{thm:semi_decidable} does not require the initial set $\init$ and the target set $\target$ to be polyhedral.
It is sufficient that they are $\omega$-regular to effectively decide (non)emptiness of the intersection.
The class of $\omega$-regular languages is strictly more expressive than polyhedral sets.\footnote{For example, the set of natural numbers is $\omega$-regular (in the encoding used here), but it is not a polyhedron.}
Thus, our result is more general than the statement of Theorem~\ref{thm:semi_decidable}.

\subsection{Multi-mode linear plants}

In this subsection, we give an example of a plant model that falls into the class of $\omega$-regular languages.
Our example is inspired by linear hybrid automata~\cite{AlurCHH92}, which are finite state machines with constant-term ordinary differential equations (ODEs) in the modes (states) and guard conditions on the transitions.

Hybrid automata have two sources of nondeterminism: an enabled transition need not be taken (may-semantics), and multiple transitions may be enabled at the same time.
Because we have restricted ourselves to deterministic plants in this work, we need to introduce some restrictions.
First, we assume a fixed rational control period, and a transition can only be taken at the end of such a period.
Then, the solution of the ODEs is a linear map, which can be analytically computed, and our system becomes discrete-time.
Second, we require that exactly one guard is enabled, i.e., in each mode, all guards are pairwise-disjoint and their union is the universe.
To simplify the presentation, we do not include discrete updates with the transitions but note that these can easily be added.
We call the resulting model a multi-mode linear map.

\begin{definition}[Multi-mode linear map]
    A \emph{multi-mode linear map} is a tuple $\MMLM = (M, E, \idim, \cdim, \flow, \grd)$ consisting of
    a finite set~$M\subseteq \nats$ of modes, a set of edges~$E \subseteq M \times M$, input and control dimensions~$\idim$ and $\cdim$, a flow function~$\flow\colon M \to \Q^{\idim \times \idim} \times \Q^{\idim \times \cdim} \times \Q^{\idim}$ (mapping a mode to two matrices and a vector), and a guard function~$\grd\colon E \to \UPh(\idim+\cdim)$ (where $\UPh$ denotes the set of finite unions of polyhedra), satisfying
    \begin{itemize}
        \item if $(m, m') \in E$ and $(m, m'') \in E$, then $\grd(m, m') \cap \grd(m, m'') = \emptyset$, and

        \item $\bigcup_{m' \in M} \grd(m, m') = \R^{\idim+\cdim}$ for all $m \in M$.
    \end{itemize}
    The function~$f_\MMLM\colon M\times \R^{\idim+\cdim} \to M \times\R^\idim$ computed by $\MMLM$ is defined as
    \[
    f_\MMLM(m,x_1,\dots,x_\idim, u_1,\dots,u_\cdim) = (m',x_1',\dots,x_\idim')
    \]
    where $\flow(m) = (A,B,c)$,
    \[
    (x_1',\dots,x_\idim') = A\cdot (x_1,\dots,x_\idim)^T + B\cdot( u_1,\dots,u_\cdim)^T + c,
    \]
    and $m'$ is the unique mode such that 
    \[
    (x_1',\dots,x_\idim',u_1,\dots,u_\cdim) \in \grd(m, m').
    \]
\end{definition}

Note that the first component of inputs for $f_\MMLM$ is restricted to modes of $\MMLM$, not arbitrary reals as stipulated by the definition of plants. 
However, this is not an issue as long as the initial input has a mode in the first component, as $f_\MMLM$ also returns only outputs that have a mode in the first component. 

\begin{lemma}\label{lemma:lha_NBA}
    Multi-mode linear maps are $\omega$-regular plants.
\end{lemma}

\begin{proof}[Sketch]
The following operations can be implemented by NBAs~\cite{SalzerABL22}:
\begin{itemize}
    \item Multiplication of real inputs with constants in $\Q$ and addition of reals. These two operations allow us to compute the output~$(x_1',\dots,x_\idim')$ from $(x_1,\dots,x_\idim)$ and $(u_1,\dots,u_\cdim)$. 
    \item Checking membership of a vector of reals in a fixed polyhedron. This allows us to compute the next mode~$m'$ from the current mode~$m$, the current state $(x_1,\dots,x_\idim)$, and the current input $(u_1,\dots,u_\cdim)$, as $m'$ is determined by the membership of $(x_1,\dots,x_\idim, u_1,\dots,u_\cdim)$ in a finite union of polyhedra.
\end{itemize}
This allows us to build an NBA that accepts the graph of $f_\MMLM$ for every given multi-mode linear map~$\MMLM$.
    \qed
\end{proof}

\begin{corollary}
The NNCS reachability problem is semi-decidable when the plant is restricted to multi-mode linear maps.
\end{corollary}

\section{Conclusion}\label{sec:conclusion}

In this paper, we studied the reachability problem for dynamical systems controlled by deep neural networks.
We showed that, for the common ReLU activations, the problem is undecidable even when the plant is trivial and the network is restricted to integral weights and a singleton initial set; furthermore, we can either fix the input and output dimensions to~$3$ and the number of hidden layers to~$6$, or use a single hidden layer.
We then turned to the question when the problem can be semi-decided; here we extended a recent encoding of neural networks in Büchi automata and showed that $\omega$-regular plants as well as input and target sets are sufficient for a semi-decision procedure; as an example, we demonstrated that a model akin to linear hybrid automata is $\omega$-regular.

\section*{Acknowledgments}

We thank the participants of AISoLA 2023 for suggesting to study the NNCS reachability problem with one hidden layer.

This research was partly supported by the Independent Research Fund Denmark under reference number 10.46540/3120-00041B, DIREC - Digital Research Centre Denmark under reference number 9142-0001B, and the Villum Investigator Grant S4OS under reference number 37819.

\bibliographystyle{splncs04}
\bibliography{bibliography}

\end{document}